\documentclass[11pt]{article}
\usepackage[utf8]{inputenc}
\usepackage[T1]{fontenc}
\usepackage{geometry}
\usepackage{setspace}
\usepackage{mathtools}
\usepackage{graphicx, epstopdf}
\usepackage[arrow, matrix]{xy}
\usepackage{xcolor}
\usepackage{cite}
\usepackage{float}
\usepackage{amssymb, amsmath, amsthm}

\newtheorem{thm}{Theorem}
\newtheorem{defn}{Definition}
\newtheorem{lemma}{Lemma}
\newtheorem{pro}{Proposition}
\newtheorem{rk}{Remark}

\title{Ising Models with Hidden Markov Structure:\\ Applications to Probabilistic Inference in Machine Learning}

\date{}

\begin{document}
	
	\maketitle
	
	\begin{center}
		\textbf{F. Herrera} \\
		Department of Computer Science and Artificial Intelligence, University of Granada, E-18071 Granada, Spain \\
		Email: \texttt{herrera@decsai.ugr.es}
		\vspace{1em}
		
		\textbf{U. A. Rozikov} \\
		 V.I. Romanovskiy Institute of Mathematics, Uzbekistan Academy of Sciences, 9, Universitet str., 100174, Tashkent, Uzbekistan; \\
		 National University of Uzbekistan, 4, Universitet str., 100174, Tashkent, Uzbekistan; \\
		Karshi State University, 17, Kuchabag str., 180119, Karshi, Uzbekistan. \\
		Email: \texttt{rozikovu@yandex.ru} \\
		\textbf{(Corresponding author)}
		\vspace{1em}
		
		\textbf{M. V. Velasco} \\
		Departamento de Análisis Matemático, Facultad de Ciencias, Universidad de Granada, 18071 Granada, Spain \\
		Email: \texttt{vvelasco@ugr.es}
	\end{center}
	
	\vspace{1em}
	
	\begin{abstract}
In this paper, we investigate  tree-indexed Markov chains (Gibbs measures) defined by a Hamiltonian that couples two Ising layers: hidden spins \(s(x) \in \{\pm 1\}\) and observed spins \(\sigma(x) \in \{\pm 1\}\) on a Cayley tree. The Hamiltonian incorporates Ising interactions within each layer and site-wise emission couplings between layers, extending hidden Markov models to a bilayer Markov random field. 
	 Specifically, we explore translation-invariant Gibbs measures (TIGM) of this Hamiltonian on Cayley trees.
		
		Under certain explicit conditions on the model's parameters, we demonstrate that there can be up to three distinct TIGMs. Each of these measures represents an equilibrium state of the spin system. These measures provide a structured approach to inference on hierarchical data in machine learning. They have practical applications in tasks such as denoising, weakly supervised learning, and anomaly detection. The Cayley tree structure is particularly advantageous for exact inference due to its tractability.
	\end{abstract}
	
	\vspace{1em}
	
	\section{Introduction}

	The history of applying statistical physics methods in machine learning spans decades and continues to grow (e.g., \cite{Ar}, \cite{Bah}, \cite{Beh}, \cite{Gem}, \cite{Krz}, \cite{Krza}, \cite{Ni}, \cite{Le}), with new insights constantly emerging. The intersection of statistical mechanics and machine learning has led to the development of powerful algorithms for inference, optimization, and understanding complex data. As machine learning models grow in complexity, the use of statistical physics concepts will likely play an increasingly important role in both theoretical developments and practical applications.
	 	 
	 In this paper, we use statistical physics methods (energy based learning) to explore a Hamiltonian model that combines the Ising interaction between hidden binary spins with a data-dependent term that links these hidden spins to observed variables. This hybrid model enables us to examine probabilistic inference within the context of hierarchical data in machine learning applications. Specifically, we study translation-invariant Gibbs measures (TIGMs) (see \cite{BR}, \cite{Ge}, \cite{Qa}, \cite{QaM}, \cite{Ro}, \cite{Rbp} for the theory of Gibbs measures and its applications) of this Hamiltonian on Cayley trees, a class of graphs that allow for efficient and tractable exact inference.
	 
	 The Ising model, a well-established framework in statistical physics, models spin configurations where spins interact with their neighbors. The extension of this model to include hidden variables and observed noisy measurements makes it particularly applicable to problems in machine learning, such as denoising, weakly supervised learning, and anomaly detection. By considering both the Ising interactions and the hidden Markov structure, we derive a framework that allows for inference over hierarchical structures where data are only partially observed or are corrupted by noise.
	 
	 The motivations behind this study are twofold. First, hierarchical models have gained prominence in machine learning due to their ability to represent complex, multi-level structures in data. A hidden Markov models (HMMs) have been widely used to model such hierarchical systems, where the data are assumed to be generated by underlying hidden states (see \cite{Gh}, \cite{We} and the references therein). However, inference in these models, especially when dealing with spatial or temporal correlations (as seen in image or sequence data), often becomes intractable. By introducing the Ising interactions between the hidden states, we aim to explore more effective ways to model these dependencies while maintaining computational feasibility.
	 
	 Second, probabilistic graphical models have shown tremendous success in machine learning, particularly in areas such as generative modeling, denoising, and anomaly detection (see \cite{Bis}, \cite{Pe}, \cite{Wa}). However, a key challenge remains: how to efficiently perform inference when the underlying data are noisy or missing. The incorporation of a hidden Markov structure within the Ising framework provides a principled way to handle noisy or incomplete observations, as the model's energy function captures both the relationship between hidden and observed variables as well as their dependencies. 
	 
	 The key innovation in this work is the use of Cayley trees, a simple yet powerful structure for exact inference, which is often intractable on general graphs.
	 The use of Cayley trees is particularly significant because it simplifies the computational complexity associated with inference. The tree structure allows for exact inference, as opposed to the approximate methods typically required on general graphs. This advantage is particularly important when dealing with large datasets, where approximate methods can lead to inaccuracies or prohibitively high computational costs.

	 The rest of the paper is organized as follows: In Section 2, we introduce the necessary preliminaries on Cayley trees and define the model. Section 3 discusses the Hamiltonian that governs the system, along with its interpretation and relation to the context of HMMs and machine learning. In Section 4, we derive the Gibbs measures and analyze their properties. Finally, Section 5 presents applications to real-world machine learning problems, illustrating how the model can be used for denoising, weakly supervised learning, and anomaly detection.

	\section{Preliminaries}
	
	{\bf Cayley tree.}
		A Cayley tree $\Gamma^k=(V, L)$ (where $V$ is the set of vertices and $L$ is the set of edges $\langle x, y\rangle$, $x,y\in V$) with branching factor $ k\geq 1 $ is a connected infinite graph, every vertex of which has exactly $ k+1 $ neighbors.  The graph $\Gamma^k$ is acyclic, meaning it has no loops or cycles (see \cite[Chapter 1]{Ro}).  
		
		Fix a vertex $x^0\in V$, interpreted as the \emph{root} of
		the tree. We say that $y\in V$ is a \emph{direct successor}  of $x\in
		V$ if $x$ is the penultimate vertex on the unique path leading from
		the root $x^0$ to the vertex $y$; that is,
		$d(x^0,y)=d(x^0,x)+1$ and $d(x,y)=1$. The set of all direct
		successors of $x\in V$ is denoted by $S(x)$.
		
		For a fixed $x^0\in V$ we set $ W_n = \ \{x\in V\ \ | \ \ d (x,
		x^0) =n \}, $
		\begin{equation}\label{lp*}
			V_n = \ \{x\in V\ \ | \ \ d (x, x^0) \leq n \},\ \ L_n = \ \{l =
			\langle x, y\rangle \in L \  | \ x, y \in V_n \}.
		\end{equation}
		For $x\in W_{n}$ the set $S(x)$ then has the form
		\begin{equation}\label{sx}
			S(x)=\{y\in W_{n+1}: \langle x, y\rangle\in L\}.
		\end{equation}
	
	{\bf Hamiltonian for an HMM: }	Hidden Markov model (HMM) can be described in statistical mechanics terms using a Hamiltonian that captures the joint distribution of hidden states and observations (see \cite{Gh}, \cite{We}). In this context, the hidden states $s(x)$ play the role of spins, and the observations $\sigma(x)$ are the emitted (visible) symbols or measurements.
		
	Let us define a general Hamiltonian for the HMM on a graph (e.g., a Cayley tree or lattice), where
	 $x \in V$ an vertex, $s(x) \in \mathcal{S}$ are hidden states (e.g., spins),  $\sigma(x) \in \Sigma$ are observed symbols.
		
		Then the Hamiltonian is:
	\begin{equation}\label{hhm}		H(s, \sigma) = - \sum_{\langle x, y \rangle} E(s(x), s(y)) - \sum_{x} p(s(x)\mid \sigma(x)),
	\end{equation}
	here interaction term (hidden states):
		$
		-\sum_{\langle x, y \rangle} E(s(x), s(y))
		$
	represents the Markov property: hidden states interact locally (e.g., adjacent in a graph).
 $E(s(x), s(y))$ is the interaction energy between neighboring hidden states.
		
Emission term (observation likelihood): $	-\sum_{x} p(s(x)\mid \sigma(x))$.  This captures the probability of emitting observation $\sigma(x)$ given the hidden state $s(x)$. Here
 $p(s\mid \sigma) = \log P(\sigma \mid s)$ is a log-likelihood function of the emission probability.

{\bf Our model}:				
	
In this paper we consider Ising spins $s(x)\in I=\{-1,1\}$ (i.e. $\mathcal S=\Sigma=I$) which are assigned to the vertices of a Cayley tree. 
	
A configuration $s$ on $V$ is then defined
as a function $x\in V\mapsto s(x)\in I$;
the set of all configurations is $\Omega:=I^V$.
		
We consider a hidden configuration $s\in \Omega$ and an observed configuration $\sigma\in \Omega$ and formulate a Hamiltonian that depends on both configurations of the Ising model.
		
	Now we  write a Hamiltonian that describes both the Ising model with interactions between spins and the relationship between the spins and the observations.
		
		The considered  Hamiltonian, $ H(s, \sigma)$, which   
		defines the energy of a given configuration pair $ (s, \sigma)$, incorporating both the Ising interactions and the relationship between the hidden states $s$ and the observations $\sigma$ given as follows:
		
	\begin{equation}\label{Ha}	
		H(s, \sigma) = -J \sum_{\langle x,y \rangle} (s(x)s(y)-\sigma(x)\sigma(y)) - \sum_{x\in V} p(\sigma(x) | s(x))
		\end{equation}
	Here, $J\in \mathbb R$. 
	\begin{rk} 		
		In a more general setting, the Hamiltonian (\ref{Ha}) can be written as
		\begin{equation}\label{Hag}
			H(s, \sigma) = -J_1 \sum_{\langle x,y \rangle} s(x)s(y) + J_2 \sum_{\langle x,y \rangle} \sigma(x)\sigma(y) - \sum_{x \in V} p(\sigma(x) \mid s(x)),
		\end{equation}
		where $J_1, J_2 \in \mathbb{R}$. In this case, one can also prove an analogue of Theorem 1 (see below); however, the corresponding system of equations becomes more complicated due to the presence of an additional parameter. For this reason, we restrict ourselves in this paper to the simpler case $J_1 = J_2 = J$. As will be seen below, even in this simplified case, the analysis of the resulting system of equations is nontrivial.
			
	\end{rk}
	
		\begin{rk} Comparing  Hamiltonians (\ref{hhm}) and (\ref{Ha}) one can see that (\ref{Ha}) has (additional) the antiferromagnetic coupling term $J \sum_{\langle x,y \rangle} \sigma(x)\sigma(y)$, coupling the observable $\sigma$ spins.  
	\begin{itemize}
	\item[1.] The first part, $ s(x)s(y)$, as in (\ref{hhm}), represents the natural interaction between spins on neighbors $ x$ and $ y$ in the Ising model (favoring alignment of spins with coupling strength $J$). The second part, $\sigma(x)\sigma(y)$, introduces a discrepancy between the observed configurations, which could indicate how much the observations diverge from the hidden states. The difference $(s(x)s(y) - \sigma(x)\sigma(y))$ essentially {\bf penalizes mismatches} between the hidden spin configuration and the observed configuration in the Ising framework.\\

\item[2.]	The second term in (\ref{Ha}) is similar to second term of (\ref{hhm}) reflects the probabilistic nature of the observations. It models how likely the observed spin $ \sigma(x)$ is, given the true (hidden) spin $s(x)$. The probability $ p(\sigma(x) | s(x))$ can be thought of as a noise model for the observation, where different types of distributions (e.g., Gaussian, Bernoulli) can be chosen depending on the nature of the observations.
	
\item[3.] 
The Hamiltonian (\ref{Ha}) does {\bf not} define an HMM because it includes a direct interaction term for observations ($\sigma(x)\sigma(y)$), which violates the fundamental HMM property that observations should be conditionally independent given the hidden states. Instead, (\ref{Ha}) describes a more general coupled system, deviating from the standard HMM framework defined by (\ref{hhm}). While standard HMMs enforce conditional independence of observations given hidden states, (\ref{Ha}) introduces observation correlations, generalizing to a bilayer Markov chain indexed by the tree. Therefore the model constructed by (\ref{Ha}) can be called as "HMM with Ising-structured observations". 
	\end{itemize}
	\end{rk}

\begin{rk}  In \cite{En}, a variant of (\ref{Ha})—lacking the direct observation coupling term $J \sum_{\langle x,y \rangle} \sigma(x)\sigma(y)$ —was analyzed. Unlike our approach, \cite{En} assumes no interactions between observed spins, instead modeling noisy observations of coupled hidden spins via site-wise (time-dependent) couplings between $\sigma(x)$ and $s(x)$. The authors studied homogeneous Gibbs measures (HGMs) on a Cayley tree under infinite-temperature Glauber dynamics, demonstrating distinct non-Gibbsian behaviors in the intermediate state (e.g., free-boundary Gibbs state) versus extremal states ($\pm$). While \cite{En}'s model holds relevance for machine learning (via HMMs), our work focuses exclusively on characterizing HGMs for Hamiltonian (\ref{Ha}) without Glauber dynamical evolution. 
	 
{\bf Motivation for adding the second term to Hamiltonian (\ref{Ha}):}	The inclusion of the antiferromagnetic term in Hamiltonian (\ref{Ha}) serves not only to facilitate the study of translation-invariant Gibbs measures but also offers significant advantages for machine learning applications. Specifically, we propose using the expression $(s(x)s(y) - \sigma(x)\sigma(y))$—which represents the difference between the Ising model energies of the hidden and observed configurations—as a {\bf loss function} for machine learning tasks (see \cite{Le} for the representation of a Hamiltonian as an energy-loss function). This formulation is motivated by its ability to directly quantify the discrepancy in pairwise spin interactions, a key feature of Ising-type models. Minimizing this loss function encourages the hidden configuration to replicate the statistical behavior of the observed one in terms of energy, thereby leading to more interpretable and physically meaningful learning outcomes—especially in settings where spin correlations are of fundamental importance.
\end{rk}
	
	{\bf Interpretation in Machine Learning:}
		
	- Energy minimization:  
	The Hamiltonian $ H(s, \sigma)$ essentially represents the total energy of a configuration. In machine learning, the goal would typically be to minimize this energy function (or equivalently maximize the likelihood of the observed data given the hidden states). 
	
	- Inference:  
	Given observed data $\sigma$, the task might be to infer the most likely hidden configuration $ s $ that generated the observations, which is a typical problem in HMMs and probabilistic graphical models. The Hamiltonian provides a framework for this by considering both the structure of the hidden system and the noise in the observations.
	
 This setup aligns with standard approaches in machine learning, where the goal is often to infer hidden structures (like the spins $s$) from noisy or indirect observations (like $\sigma$), with the Hamiltonian serving as the energy function to minimize during the inference process. Below we construct some Gibbs measures of the Hamiltonian (\ref{Ha}) and then use these measures to solve above mentioned problems of Machine learning.  
	
		\section{The compatibility of measures}
	
	Define a finite-dimensional distribution of a probability measure $\mu$ in the volume $V_n$ as
	\begin{equation}\label{p*}
		\mu_n(s_n,\sigma_n)=Z_n^{-1}\exp\left\{-\beta H_n(s_n,\sigma_n)+\sum\limits_{x\in W_n}h_{s(x),\sigma(x),x}\right\},
	\end{equation}
	where $\beta=1/T$, and  $T>0$ is the temperature,  $Z_n^{-1}$ is the normalizing factor,
	\begin{equation}\label{hx}	\{h_x=(h_{-1,-1,x}, h_{-1,1,x}, h_{1,-1,x}, h_{1,1,x} )\in \mathbb R^{4}, x\in V\}
	\end{equation}
	is a collection of vectors and
	$$H_n(s_n, \sigma_n)=-J\sum\limits_{\langle x,y\rangle\in L_n}
(s(x)s(y)-\sigma(x)\sigma(y)) - \sum_{x\in V_n} p(\sigma(x) | s(x)).$$
	
\begin{defn}	We say that the probability distributions (\ref{p*}) are compatible if for all
	$n\geq 1$ and  $s_{n-1}\in I^{V_{n-1}}$, 
	$\sigma_{n-1}\in I^{V_{n-1}}$:
	\begin{equation}\label{p**}
		\sum\limits_{(w_n, \omega_n)\in I^{W_n}\times I^{W_n}}\mu_n(s_{n-1}\vee w_n, \sigma_{n-1}\vee \omega_n)=\mu_{n-1}(s_{n-1},\sigma_{n-1}).
	\end{equation}
	Here symbol $\vee$ denotes the concatenation (union) of the configurations.
	\end{defn} 

	In this case, by Kolmogorov's theorem (\cite[p.251]{FV}) there exists a unique measure $\mu$ on $I^V\times I^V$ such that,
	for all $n$ and  $s_n, \sigma_n \in I^{V_n}$,
	$$\mu(\{(s, \sigma)|_{V_n}=(s_n,\sigma_n)\})=\mu_n(s_n, \sigma_n).$$
	Such a measure is called a {\it splitting Gibbs measure} (SGM) corresponding to the Hamiltonian (\ref{Ha}) and vector-valued functions (\ref{hx}).
	\begin{rk}
In this paper, we exclusively study splitting Gibbs measures and thus omit the qualifier `splitting' for brevity.
	\end{rk}
	The following theorem describes conditions on $h_x$ guaranteeing compatibility of $\mu_n(s_n,\sigma_n)$.
	
	\begin{thm}\label{ep} Probability distributions
		$\mu_n(s_n,\sigma_n)$, $n\in \mathbb N$, in
		(\ref{p*}) are compatible iff for any $x\in V\setminus \{x^0\}$
		the following equation holds:
		\begin{equation}\label{p***}
			z_{-1, 1, x}=\prod_{y\in S(x)}
			{\theta+ az_{-1,1,y}+bz_{1,-1,y}+\theta^{-1}cz_{1,1,y}\over 1+ \theta az_{-1,1,y}+\theta^{-1} bz_{1,-1,y}+cz_{1,1,y}},
		\end{equation}
		\begin{equation}\label{ou}	
				z_{1, -1, x}=\prod_{y\in S(x)}
			{\theta^{-1}+ az_{-1,1,y}+bz_{1,-1,y}+\theta cz_{1,1,y}\over 1+ \theta az_{-1,1,y}+\theta^{-1} bz_{1,-1,y}+cz_{1,1,y}},\end{equation}
			\begin{equation}\label{zu}	
			z_{1, 1, x}=\prod_{y\in S(x)}
			{1+\theta^{-1} az_{-1,1,y}+\theta bz_{1,-1,y}+ cz_{1,1,y}\over 1+ \theta az_{-1,1,y}+\theta^{-1} bz_{1,-1,y}+cz_{1,1,y}},\end{equation}
		where
		\begin{equation}\label{zt}		
		 z_{\epsilon, \delta, x}=\exp(h_{\epsilon, \delta, x}-h_{-1, -1, x}), \, 	\epsilon, \delta=-1,1, \ \
			\theta=\exp(2J\beta),\end{equation}
	$$a=\exp[\beta(p(-1|1)-p(-1|-1))],$$
	$$b=\exp[\beta(p(1|-1)-p(-1|-1))],$$  $$c=\exp[\beta(p(1|1)-p(-1|-1))],$$
		and $S(x)$ is the set of direct successors of $x$.
	\end{thm}
	\begin{proof}  Below we use the following equalities:
		$$V_n=V_{n-1}\cup W_n, \ \ W_n=\bigcup_{x\in W_{n-1}}S(x).$$
		
		{\sl Necessity.}  Assume that (\ref{p**}) holds, we shall prove (\ref{p***}). Substituting (\ref{p*}) into
		(\ref{p**}), obtain that  for any configurations $(s_n,\sigma_{n-1})$: $x\in V_{n-1}\mapsto (s_n(x), \sigma_{n-1}(x))\in I\times I $:
		$$			\frac{Z_{n-1}}{Z_n}\sum\limits_{(w_n, \omega_n)\in I^{W_n}\times I^{W_n}}
		\exp\left(\sum\limits_{x\in W_{n-1}}\sum\limits_{y\in S(x)} (J\beta (s_{n-1}(x)w_n(y)-\sigma_{n-1}(x)\omega_n(y))\right.$$
		\begin{equation}\label{puu}	\left.+	\beta p(\omega_n(y)|w_n(y))+
		h_{w_n(y),\omega_n(y),y})\right)
				=		\exp\left(\sum\limits_{x\in
				W_{n-1}}h_{s_{n-1}(x),\sigma_{n-1}(x),x}\right).
		\end{equation}
		
		From (\ref{puu}) we get:
		$${Z_{n-1}\over Z_n}\sum\limits_{(w_n, \omega_n)\in I^{W_n}\times I^{W_n}}
		\prod_{x\in W_{n-1}}\prod_{y\in S(x)} \exp\,(J\beta (s_{n-1}(x)w_n(y)-\sigma_{n-1}(x)\omega_n(y))$$ $$
	+	\beta p(\omega_n(y)|w_n(y))+
			h_{w_n(y),\omega_n(y),y}) = \prod_{x\in W_{n-1}} \exp\,(h_{s_{n-1}(x),\sigma_{n-1}(x),x}). $$
		Fix $x\in W_{n-1}$ and rewrite the last equality for
		$(s_{n-1}(x), \sigma_{n-1}(x))=(\epsilon, \delta)$, $\epsilon, \delta=-1,1$, keeping the configurations unchanged on $W_{n-1}\setminus \{x\}$. Then, by dividing both sides of each equation by the corresponding sides of the equation for the case  $(s_{n-1}(x), \sigma_{n-1}(x))=(-1, -1)$,  we obtain:
		\begin{equation}\label{als}\prod_{y\in S(x)}\frac{\sum\limits_{(j, u)\in I\times I}\exp\,(J\beta (\epsilon j-\delta u)+\beta p(u|j)+h_{j,u,y})}{\sum\limits_{(j, u)\in I\times I}\exp\,(-J\beta(j-u)+\beta p(u|j)+h_{j,u,y})}= \exp\,(h_{\epsilon, \delta, x}-h_{-1, -1, x}),\end{equation}
		where $\epsilon, \delta=-1,1.$
		
		Now by using notations (\ref{zt}), from (\ref{als}) we get  (\ref{p***})-(\ref{zu}).
		Note that $z_{-1, -1, x}\equiv 1$.
		
		{\sl Sufficiency.} Suppose that (\ref{p***})-(\ref{zu}) hold. It is equivalent to
		the representations
		\begin{equation}\label{pru}
			\prod_{y\in S(x)}\sum\limits_{(j, u)\in I\times I}\exp\,(J\beta (\epsilon j-\delta u)+\beta p(u|j)+h_{j,u,y})= a(x)\exp\,(h_{\epsilon,\delta,x}), \ \ \epsilon, \delta =-1,1
		\end{equation} for some function $a(x)>0, x\in V.$
		We have
		$$
		{\rm LHS \ of \ (\ref{p**})}=\frac{1}{Z_n}\exp(-\beta H(s_{n-1},\sigma_{n-1}))\times
		$$
		\begin{equation}\label{pru1}
			\prod_{x\in W_{n-1}} \prod_{y\in S(x)}\sum\limits_{(\epsilon,u)\in I\times I}\exp\,(J\beta (s_{n-1}(x)\epsilon-\sigma_{n-1}(x)u)+\beta p(u|\epsilon)+h_{\epsilon,u,y}).\end{equation}
		
		Substituting (\ref{pru}) into (\ref{pru1}) and denoting $A_n=\prod_{x \in W_{n-1}} a(x)$,
		we get
		\begin{equation}\label{pru2}
			{\rm RHS\ \  of\ \  (\ref{pru1}) }=
			\frac{A_{n-1}}{Z_n}\exp(-\beta H(s_{n-1},\sigma_{n-1}))\prod_{x\in W_{n-1}}
			\exp(h_{s_{n-1}(x),\sigma_{n-1}(x),x}).\end{equation}
		
		Since $\mu_n$, $n \geq 1$ is a probability, we should have
		$$ \sum\limits_{(s_{n-1}, \sigma_{n-1})\in I^{V_{n-1}}\times I^{V_{n-1}}} \ \ \sum\limits_{(w_n, \omega_n)\in I^{W_n}\times I^{W_n}}\mu_n(s_{n-1}\vee w_n, \sigma_{n-1}\vee \omega_n) = 1. $$
		
		Hence from (\ref{pru2})  we get $Z_{n-1}A_{n-1}=Z_n$, and (\ref{p**}) holds.
		
	\end{proof}

	From Theorem \ref{ep}, it follows that for any $ z = \{ z_{\epsilon, \delta, x}, \ x \in V \} $ that satisfies the system of functional equations (\ref{p***})-(\ref{zu}), there exists a unique splitting Gibbs measure (SGM) $ \mu $, and conversely. Therefore, the core problem of describing Gibbs measures reduces to solving the system of functional equations (\ref{p***})-(\ref{zu}).
	
	However, solving this system is particularly challenging due to its non-linear nature, multidimensional structure, and the fact that the unknown functions are defined on a tree. Even the task of determining all constant functions (which are independent of the tree's vertices) is complex. In this paper, we present a class of such constant solutions and explore how the corresponding Gibbs measures can be applied in machine learning.

	\section{Translation-invariant Gibbs Measures}
	
	In this section, we focus on translation-invariant Gibbs measures, which correspond to solutions of the form:
\begin{equation}\label{ze}
	z_{\epsilon, i, x} \equiv z_{\epsilon, i}, \quad \text{for all} \ x \in V.
\end{equation}
	
	Denote
	$$u=az_{-1,1}, \ \ v=bz_{1,-1}, \ \ w=cz_{1,1}.$$
	
	Substituting this into equation (\ref{p***})-(\ref{zu}) gives:

		\begin{equation}\label{ABC}
			\begin{array}{lll}
				u=a\left({\theta+u+v+\theta^{-1}w\over 1+\theta u+\theta^{-1}v+w}\right)^k\\[2mm]
					v=b\left({\theta^{-1}+u+v+\theta w\over 1+\theta u+\theta^{-1}v+w}\right)^k\\[2mm]
					w=c\left({1+\theta^{-1}u+\theta v+w\over 1+\theta u+\theta^{-1}v+w}\right)^k.
					\end{array}
		\end{equation}

	
\subsection{ Case:	 $a=b=c=1$}  In this case write the system (\ref{ABC}) as an equation of fixed point $F(t)=t$, where $t=(u,v,w)$ and the operator $F:\mathbb R^3_+\to \mathbb R^3_+$ is given by  
	\begin{equation}\label{F}
	\begin{array}{lll}
		u'=\left({\theta+u+v+\theta^{-1}w\over 1+\theta u+\theta^{-1}v+w}\right)^k\\[2mm]
		v'=\left({\theta^{-1}+u+v+\theta w\over 1+\theta u+\theta^{-1}v+w}\right)^k\\[2mm]
		w'=\left({1+\theta^{-1}u+\theta v+w\over 1+\theta u+\theta^{-1}v+w}\right)^k.
	\end{array}
\end{equation}
Consider sets 
$$I_1=\{t=(u,v,w)\in \mathbb R^3_+: u=1, v=w\}.$$
$$I_2=\{t=(u,v,w)\in \mathbb R^3_+: u=w,  v=1\}.$$
$$I_3=\{t=(u,v,w)\in \mathbb R^3_+: u=v, w=1\}.$$
The following lemma is straightforward:
\begin{lemma}
	Each set $I_i$, $i=1,2,3$ is invariant with respect to $F$, i.e. $F(I_i)\subset I_i$.
\end{lemma} 
Now we shall find fixed points of $F$ on each invariant set $I_i$.
Fixed points of restricted operator $F$ on  each of the invariant sets  will satisfy  an equation  the form 
\begin{equation}\label{kap}
	x= f(x,\gamma):=\left({1+\gamma x\over \gamma+x}\right)^k.
\end{equation}
 \begin{lemma}\label{m0} Let $k\geq 2$. The equation (\ref{kap}) has a unique solution $x=1$ if $0<\gamma\leq {k+1\over k-1}$ and it has three solutions if $\gamma> {k+1\over k-1}$.
 \end{lemma}
 \begin{proof} This is known see \cite{HOR}. Here we give the proof because we will use this lemma several times.
 	Note that $x=1$ is a solution of (\ref{kap}) independently on the parameters.
 	
 	{\it Case: $\gamma<1$}. In this case (\ref{kap}) has a unique solution, because $f(x,\gamma)$ is decreasing in $(0, +\infty)$.
 	
 	{\it Case: } $\gamma>1$. Denote $u=\sqrt[k]{x}$.
 	Rewrite (\ref{kap}) as
 	\begin{equation}\label{fd} u^{k+1}-\gamma u^k+\gamma u-1=0.
 	\end{equation}
 
 	By  Descartes rule (see\footnote{The Descartes rule states that if the nonzero terms of a single-variable polynomial with real coefficients are ordered by descending variable exponent, then the number of positive roots of the polynomial is either equal to the number of sign changes between consecutive (nonzero) coefficients, or is less than it by an even number. A root of multiplicity $n$ is counted as $n$ roots. 	In particular, if the number of sign changes is zero or one, the number of positive roots equals the number of sign changes.} \cite{Pra}, p.28), the equation (\ref{fd}) has up to 3 positive roots. We show  that, under suitable condition on parameters, it has exactly three roots.
 	
 	The function $f(x,\gamma)$ for $x>0$ is increasing and bounded. We have
 	\begin{equation}\label{sr4.5}
 		{d\over dx}f(1,\gamma)=f'(1,\gamma)=k{\gamma-1\over \gamma+1}>0
 	\end{equation}
 	and
 	$$f'(1,\gamma)=1, \ \ \mbox{gives} \ \ \gamma_c=
 	{k+1\over k-1}.$$
 	If	$0<\gamma\leq \gamma_c$ then $f'(1,\gamma)<1$, the solution $x=1$	is a stable fixed point of the map $f(x,\gamma)$, and $\lim_{n\to\infty}f^{(n)}(x,\gamma) = 1$, for any $x > 0$.
 	Here, $f^{(n)}$ is the $n$th iterate
 	of the above map $f(x,\gamma)$. Therefore, $1$ is the unique positive solution.
 	On the other hand, under  $\gamma> \gamma_c$, the fixed point $1$ is unstable.
 	Iterates $f^{(n)}(x,\gamma)$ remain for $x>1$, monotonically increasing  and hence converge to a limit, $z^*\geq 1$ which solves
 	(\ref{kap}). However, $z^*>1$ as $1$ is unstable. Similarly, for  $x<1$ one gets a solution  $0<z_*<1$.
 	This completes the proof.
 \end{proof}
\subsubsection{Case: $I_1$} Restricting (\ref{F}) to the set $I_1$, the equation $t=F(t)$ becomes 
\begin{equation}\label{I1}
		v=\left({\theta^{-1}+1+(1+\theta)v\over 1+\theta +(1+\theta^{-1})v}\right)^k= \left({1+\theta v\over \theta+v}\right)^k.
		\end{equation}
Applying Lemma \ref{m0} to (\ref{I1}) with $\gamma=\theta$ implies that\\
  
{\bf Result 1.} The operator $F$ on invariant set $I_1$ has the unique fixed point $\mathbf 1:=(1,1,1)$, i.e., $u=v=w=1$ if $0<\theta\leq {k+1\over k-1}$ and it has three fixed points  $\mathbf 1$, $(1, v_i, v_i)$, $i=1,2$ if $\theta> {k+1\over k-1}$.

For example, if $k=2$ then   we have $v_i$, $i=1,2$ are
\begin{equation}\label{k23}
	\begin{array}{ll}
v_1:=v_1(\theta)=\frac{1}{2}(\theta^2 - 2\theta - 1 - (\theta-1)\sqrt{(\theta+1)(\theta-3)}),\\[3mm]
v_2:=v_2(\theta)=\frac{1}{2}(\theta^2 - 2\theta - 1 + (\theta-1)\sqrt{(\theta+1)(\theta-3)}), \ \ \theta>3
\end{array}
\end{equation}

\subsubsection{Case: $I_2$} Restricting (\ref{F}) to the set $I_2$, gives 
\begin{equation}\label{I2}
	u= \left({1+\theta^{-1} u\over \theta^{-1}+v}\right)^k.
\end{equation}
Now by  Lemma \ref{m0}, from  (\ref{I2}) with $\gamma=1/\theta$ we get\\
{\bf Result 2.} The operator $F$ on invariant set $I_2$ has the unique fixed point $\mathbf 1:=(1,1,1)$, if  $\theta\geq  {k-1\over k+1}$ and it has three fixed points  $\mathbf 1$, $(u_i, 1, u_i)$, $i=1,2$ if $0<\theta< {k-1\over k+1}$.

 \subsubsection{Case: $I_3$} Restricting (\ref{F}) to the set $I_3$, gives 
 \begin{equation}\label{I3}
 	u= \left({1+ \gamma u\over \gamma+u}\right)^k,\ \ \mbox{with} \ \ \gamma={2\over \theta+\theta^{-1}}.
 \end{equation}
Since ${2\over \theta+\theta^{-1}}<1$,  by  Lemma \ref{m0}, from  (\ref{I3}) we get\\
 {\bf Result 3.} The operator $F$ on invariant set $I_3$ has the unique fixed point $\mathbf 1:=(1,1,1)$, for any   $\theta>0$. 
 
 We summarize the above-mentioned Results 1-3, applying Theorem \ref{ep}  in the following
 \begin{thm} If $k\geq 2$, $a=b=c=1$ then for Hamiltonian\footnote{The condition $a=b=c=1$ implies $p(-1|1)=p(1|-1)=p(1|1)=p(-1|-1).$} (\ref{Ha}) there exists at least one SGM if $\theta\in ({k-1\over k+1}, {k+1\over k-1})$ and at lest three SGMs if $\theta\in (0, {k-1\over k+1}]\cup [{k+1\over k-1}, +\infty)$.
\end{thm}   
	
	\begin{rk}  Note that even when \( a = b = c = 1 \), the analysis of the system (\ref{ABC}) is quite complicated for unknowns outside the invariant sets \( I_i \). However, based on numerical (computer) analysis, we {\bf conjecture} that there are no positive solutions outside these invariant sets.
			
	\end{rk}
\begin{rk} 	
	It is well-known (\cite{Ge}, \cite{Rbp}) that each Gibbs measure defines a state of the system determined by a Hamiltonian. The existence of certain parameter values (such as $\theta_c = \frac{k-1}{k+1}$) at which the uniqueness of the Gibbs measure transitions to non-uniqueness is interpreted as a phase transition. 
	
	Phase transitions in HMMs also refer to sudden or qualitative changes in the system's behavior, often driven by the model's parameters and the number of hidden states. These transitions can manifest as a shift from predictable, stable behavior to erratic or oscillatory behavior, as well as a transition between underfitting and overfitting as the number of hidden states changes.
	
	Understanding these phase transitions is crucial for tuning HMMs and ensuring the model performs optimally for a given dataset. In the following, we will clarify this point specifically in the context of the non-uniqueness of Gibbs measures.
	\end{rk}

	\subsection{Case: $k=1$, $a=b$, $c=1$.} 
		Here we assume that 
\begin{equation}\label{lam}
	p(1|1)=p(-1|-1), \ \ p(1|-1)=p(-1|1).
\end{equation}

	In this case from (\ref{ABC}) we get 
		\begin{equation}\label{AA}
		\begin{array}{lll}
			u=a\left({\theta+u+v+\theta^{-1}w\over 1+\theta u+\theta^{-1}v+w}\right)\\[2mm]
			v=a\left({\theta^{-1}+u+v+\theta w\over 1+\theta u+\theta^{-1}v+w}\right)\\[2mm]
			w=\left({1+\theta^{-1}u+\theta v+w\over 1+\theta u+\theta^{-1}v+w}\right).
		\end{array}
	\end{equation}
\begin{pro}\label{pk1} For $k=1$, $a=b$, $c=1$ the system (\ref{ABC}) (i.e. (\ref{AA})) has a unique positive solution.\end{pro}
\begin{proof}
By subtracting the second equation from the first equation of (\ref{AA}), and then subtracting 1 from the third equation, we obtain:
  	\begin{equation}\label{A1}
		\begin{array}{lll}
			u-v={a(\theta^{-1}-\theta) \over 1+\theta u+\theta^{-1}v+w}(w-1),\\[2mm]
			w-1={\theta^{-1}-\theta \over 1+\theta u+\theta^{-1}v+w}(u-v).
			\end{array}
	\end{equation}
	From (\ref{A1}) we get 
\begin{equation}\label{vu} u=v, \ \ w=1 \end{equation}
or
	\begin{equation}\label{uvw}
u\ne v, \ \  w\ne 1, \ \  (1+\theta u+\theta^{-1}v+w)^2=a(\theta^{-1}-\theta)^2.
\end{equation}

{\bf Subcase:}
Consider the case (\ref{vu}), then system of equations (\ref{ABC}) is reduced to 
$$u=a\left({\theta+\theta^{-1}+2u\over 2+(\theta +\theta^{-1})u}\right),$$
that can be rewritten as 
$$u^2+(1-a)\Theta u-a=0, \ \ \mbox{with} \ \ \Theta={2\over \theta+\theta^{-1} }.$$ 
We are interested in positive solutions $u$.
It is easy to see that for any $a>0$ and $\theta>0$, the last equation has a unique positive solution:
	\begin{equation}\label{uA} u=u_1:={1\over 2}\left((a-1)\Theta+\sqrt{((a-1)\Theta)^2+4a}\right).
		\end{equation}

{\bf Subcase:} Let us consider the case (\ref{uvw}), then from (\ref{AA}) we get
	\begin{equation}\label{AC}
	\left\{\begin{array}{lll}
		|\theta-\theta^{-1}|u=\alpha\left(\theta+u+v+\theta^{-1}w\right)\\[2mm]
		|\theta-\theta^{-1}|v=\alpha\left(\theta^{-1}+u+v+\theta w\right)\\[2mm]
		\alpha|\theta-\theta^{-1}|w=1+\theta^{-1}u+\theta v+w
	\end{array}\right.
\end{equation}
where $\alpha=\sqrt{a}$.

Adding the first and second equations of (\ref{AC}) and adding its third equation to the equation given in (\ref{uvw}) one gets
	\begin{equation}\label{Ad}
	\left\{\begin{array}{ll}
		(2\alpha-M)s+\alpha(\theta+\theta^{-1})\tau=0\\[2mm]
		(\theta+\theta^{-1})s+(2-\alpha M)\tau=0
	\end{array}\right.
\end{equation}
where $$M= |\theta-\theta^{-1}|, \, s=u+v, \, \tau=w+1.$$
Since $\theta\ne 1$ we have 
$$	\left|\begin{array}{cc}
	2\alpha-M & \alpha(\theta+\theta^{-1})\\[2mm]
	\theta+\theta^{-1} & 2-\alpha M
\end{array}\right|=-2M(\alpha^2+1)\ne 0.
$$ Therefore, the system (\ref{Ad}) has a unique solution 
$s=0$, $\tau=0$, which gives $u=-v$, $w=-1$, but we only interested in positive solutions. Thus system (\ref{AA}) has a unique positive solution $(u_1, u_1, 1)$, where $u_1$ is given by 
(\ref{uA}).\end{proof}
	\subsection{Case: $k\geq 2$, $a=b$, $c=1$.} 
	In this case we give analysis of (\ref{ABC}). 
	 
	{\bf Subcase:} $x=y$, $z=1$. The system is reduced to 
	\begin{equation}\label{ax}
		{1\over a} x=\varphi(x, \Theta):=\left({1+\Theta x\over \Theta+x}\right)^k, \, \Theta={2\over \theta+\theta^{-1}}.\end{equation}
	\begin{pro}
		For any $k\geq 2$, $a>0$, $\theta>0$ the equation (\ref{ax}) has unique solution.
	\end{pro}
	\begin{proof}
	
	Note that $\Theta<1$, because $\theta+\theta^{-1}>2$, for $\theta>0$, $\theta\ne 1$. Consequently, the function $\varphi(x, \Theta)$ is monotone decreasing function of $x$ for all $\Theta\in (0, 1)$. But the function ${x\over a}$ in the RHS of (\ref{ax}) is increasing. Therefore, equation (\ref{ax}) has a unique solution.
	\end{proof}
	
From (\ref{ABC}), denoting $x=\sqrt[k]{u}$,
$y=\sqrt[k]{v}$ and $z=\sqrt[k]{w}$ we get
 \begin{equation}\label{AV}
	\begin{array}{lll}
		x=a\cdot {\theta+x^k+y^k+\theta^{-1}z^k\over 1+\theta x^k+\theta^{-1}y^k+z^k}\\[2mm]
		y=a\cdot {\theta^{-1}+x^k+y^k+\theta z^k\over 1+\theta x^k+\theta^{-1}y^k+z^k}\\[2mm]
		z={1+\theta^{-1}x^k+\theta y^k+z^k\over 1+\theta x^k+\theta^{-1}y^k+z^k}
	\end{array}
\end{equation}

\begin{lemma} In system (\ref{AV}) $x=y$ iff $z=1$.
\end{lemma} 
 \begin{proof}
 	Rewrite (\ref{AV}) as
 \begin{equation}\label{kk}	\begin{array}{ll}
 	x-y= {a(\theta-\theta^{-1})\over  1+\theta x^k+\theta^{-1}y^k+z^k} (1-z^k)\\[2mm]
 	1-z={(\theta-\theta^{-1})\over  1+\theta x^k+\theta^{-1}y^k+z^k} (x^k-y^k)
 		\end{array}
 	\end{equation}
 If $x=y$ then from the second equation we get $z=1$. If $z=1$ then from the first equation we get $x=y$. 
 \end{proof}

{\bf Subcase:} $x\ne y$, $z\ne 1$. For simplicity we consider the case $k=2$ and from (\ref{kk}) we obtain
$$(1+\theta x^2+\theta^{-1}y^2+z^2)^2=a(\theta-\theta^{-1})^2(x+y)(1+z).$$
But this relation does not simplify the system (\ref{AV}). It still is very complicated therefore in the Table we give some numerical results.\\
\begin{center}
\begin{tabular}{|l|l|l|l|l|l|}
	\hline
$k$ & $\theta$  & $a$ & $x$ & $y$ & $z$\\
	
	\hline
		&   &  & $1.268048128$ & $1.268048128$ & $1$\\
	$2$ & $0.1$  & $2$ & $0.2005870619$ & $1.263964993$ & $0.6570348177$\\
	&   &  & $192.3741268$ & $3.052913735$ & $152.1989357$\\
	\hline
		
	&   &  & $0.7886136007$ & $0.7886136007$ & $1$\\
	$2$ & $0.1$  & $0.5$ & $0.005198204231$ & $0.7911611517$ & $0.1586966909$\\
	&   &  & $49.85366406$ & $0.3275559308$ & $63.01328616$\\
	\hline
		&   &  & $1$ & $1$ & $1$\\
	$2$ & $0.1$  & $1$ & $0.1010204092$ & $1$ & $0.1010204092$\\
	&   &  & $98.98989796$ & $1$ & $98.98989796$\\
	\hline
		$2$ & $1.3$  & $1$ & $1$ & $1$ & $1$\\
\hline
$2$ & $1.3$  & $0.5$ & $0.5376526550$ & $0.5376526550$ & $1$\\
	\hline
\end{tabular}
\end{center}

\section{Interpretations of results and their applications in Machine learning}
\subsection{Prediction of a hidden configuration based on the observed one}
For a given sequence of observed data (configuration) $\sigma = (\sigma(x) $ over vertices $x$ in the tree, one of the main mathematical problems in HMM on the Ising model is to infer the hidden states (the spin configurations) from these observations.
This can be formalized by calculating  the conditional probability $\mu(s_n | \sigma_n)$ with condition (observed) configuration $\sigma_n$. We start by considering the joint distribution \(\mu_n(s_n, \sigma_n)\) and the Hamiltonian \(H_n(s_n, \sigma_n)\). 

 Expanding the Hamiltonian, we separate terms involving \(s(x)\) and \(\sigma(x)\). The terms involving \(\sigma(x)\) are constants when conditioning on \(\sigma_n\), so they can be factored out into the normalization constant.
The remaining terms form an effective Hamiltonian for the \(s_n\) spins, which includes the interaction between \(s\) spins, the local terms \(\beta p(\sigma(x) | s(x))\), and the boundary fields \(h_{s(x), \sigma(x), x}\).

The final conditional probability is given by:

$$
\mu(s_n | \sigma_n) = \frac{1}{Z(\sigma_n)} \exp\left( \beta J \sum_{\langle x,y \rangle} s(x)s(y) + \sum_{x \in V_n} \beta p(\sigma(x) | s(x)) + \sum_{x \in W_n} h_{s(x), \sigma(x), x} \right)$$
\begin{equation}\label{ms}
=\frac{1}{Z(\sigma_n)} \prod_{\langle x,y \rangle}\theta^{ {s(x)s(y)\over 2}} \prod_{x \in V_n} \exp(\beta p(\sigma(x) | s(x))) \prod_{x \in W_n} z_{s(x), \sigma(x), x}, 
\end{equation}
where $Z(\sigma_n)$ is the normalization factor dependent on $\sigma_n$.
\begin{rk}
It follows from formula (\ref{ms}) that to define the conditional probability $\mu(s_n | \sigma_n)$, it suffices to know the conditional probabilities on each edge. Therefore, we examine these probabilities on each edge. Since the Gibbs measures we have derived above are translation-invariant, in formula (\ref{ms}), we have $z_{s(x), \sigma(x), x} = z_{s(x), \sigma(x)}$, meaning that it does not depend on the vertex $x$, but depends on the values of configurations at the vertex.
\end{rk}
Recall (see Theorem \ref{ep}):
	$$ z_{\epsilon, \delta, x}=\exp(h_{\epsilon, \delta, x}-h_{-1, -1, x}), \, 	\epsilon, \delta=-1,1.$$
Without loss of generality we assume $h_{-1, -1, x}\equiv 0$, then for each translation-invariant solution (\ref{ze}) we have 
\begin{equation}\label{he} h_{\epsilon, \delta, x}\equiv \log z_{\epsilon, \delta}, \ \ \epsilon, \delta=-1,1.
	\end{equation}
	
 Thus, on Cayley trees, belief propagation (BP) computes marginals $ \mu(s_n|\sigma_n) $ efficiently by (\ref{ms}). The observed $ \sigma $ act as fixed boundary conditions, reducing degeneracy in hidden states.  
 
 {\bf Case: $a=b=c=1$}:
 Let us illustrate this for configurations on a fixed edge $\ell_0=\langle x, y \rangle$  and for three  distinct Gibbs measures of Result 1: $\mu_0$ corresponding to $u=v=w=1$ and $\mu_i$, $i=1,2$ corresponding to solutions $(1, v_i, v_i)$ with $v_i$ given in  (\ref{k23}):
 
 {\bf Measure $\mu_0$}: In the case of Result 1 we have the condition $a=b=c=1$ that is 
 \begin{equation}\label{pb} p(-1|1)=p(1|1)=p(1|-1)=p(-1|-1)\equiv {1\over 2}.\end{equation}
 Since $u=v=w=1$ from (\ref{he}) we get 
 $$h_{\epsilon, \delta, x}\equiv 0, \ \ \epsilon, \delta=-1,1.$$
 Consequently, for fixed edge $\ell_0=\langle x, y \rangle$  we have 
\begin{equation}\label{M0} \mu_0(s_{\ell_0} | \sigma_{\ell_0})={\theta^{{1+s(x)s(y)\over 2}}\over 2(1+\theta)}.
	\end{equation}
By this formula, one can see that with respect to measure $\mu_0$ the conditional probability does not depend on the condition (observed) configurations. Moreover, hidden configuration on the end-points of edge $\ell_0$ has  equal values with probability $\theta/(1+\theta)$ and distinct values with probability $1/(1+\theta)$.
 
  {\bf Measure $\mu_1$}: In this case we have 
\begin{equation}\label{M1} \mu_1(s_{\ell_0} | \sigma_{\ell_0})={\theta^{{1+s(x)s(y)\over 2}}z_{s(x),\sigma(x)}z_{s(y),\sigma(y)}\over \sum_{\epsilon, \delta \in \{-1, 1\}}{\theta^{{1+\epsilon \delta\over 2}}z_{\epsilon,\sigma(x)}z_{\delta,\sigma(y)}}}.\end{equation}
For solution $(1,v_i, v_i)$ this conditional probability also does not depend on condition $\sigma_{\ell_0}$.
For example, 
$$\mu_1((1,1)| \sigma_{\ell_0})={\theta v_1(\theta)^2\over \theta v_1(\theta)^2+2v_1(\theta)+\theta}$$

 {\bf Measure $\mu_2$}: For the measure $\mu_2$ similarly to the case $\mu_1$ we obtain 

\begin{equation}\label{M2}\mu_2((1,1)| \sigma_{\ell_0})={\theta v_2(\theta)^2\over \theta v_2(\theta)^2+2v_2(\theta)+\theta}
	\end{equation}

\begin{figure}[h]
	\begin{center}
		\includegraphics[width=7.3cm]{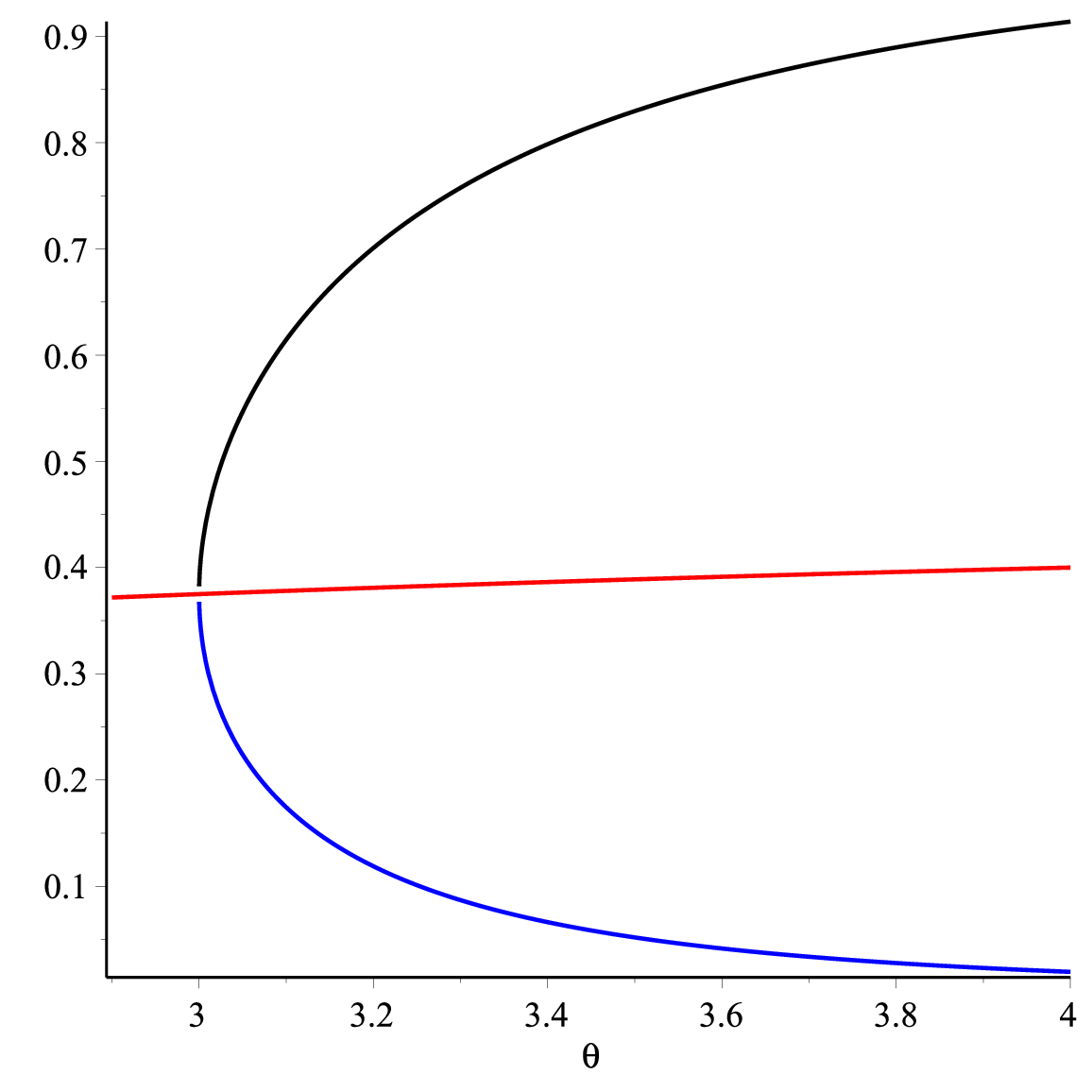}
	\end{center}
	\caption{The graph of $\mu_0((1,1)|(\sigma(x),\sigma(y)))$ (shown in the red), $\mu_1((1,1)|(\sigma(x),\sigma(y)))$ (blue) and $\mu_2((1,1)|(\sigma(x),\sigma(y)))$ (black) as functions of parameter $\theta$.}\label{mu}
\end{figure}

\begin{rk} 
	By formulas (\ref{M0})-(\ref{M2}) (see also Figure \ref{mu}), it is clear that the system described by the Hamiltonian (\ref{Ha}) with the condition (\ref{pb}) has three distinct equilibrium states:
	
	\begin{itemize}
		\item For the state corresponding to $\mu_0$, the hidden configurations at the endpoints of each edge of the tree appear with equal values with probability $\frac{\theta}{1+\theta}$ and with distinct values with probability $\frac{1}{1+\theta}$. Moreover, as $\theta \to \infty$, the hidden configuration is observed to be either all $+1$ or all $-1$, with probability $1$ for the $\mu_0$ state.
		\item For the state corresponding to $\mu_1$, the hidden configurations at the endpoints of each edge of the tree are most likely to be $+1$, with the highest probability corresponding to $\mu_1$.
		\item For the state corresponding to $\mu_2$, the hidden configurations at the endpoints of each edge of the tree are most likely to be $-1$, with the highest probability corresponding to $\mu_2$.
	\end{itemize} 
	
\end{rk}	
	
{\bf Case	$k=1$, $a=b$, $c=1$}: In this case by Proposition \ref{pk1} we have a unique Gibbs measure, denote it by $\mu^*$. 
From (\ref{ms}) we obtain (recall $u_1$ given in (\ref{uA}))
for condition $\sigma_{\ell_0}=(1,1)$:
\begin{equation}\label{ona}\mu^*((1,1)|(1,1))={\theta\over 2\theta+u_1+u_1^2}, \ \ 
\mu^*((-1,1)|(1,1))={u_1\over 2\theta+u_1+u_1^2},\end{equation}
$$\mu^*((1,-1)|(1,1))={\theta\over 2\theta+u_1+u_1^2}, \ \ 
\mu^*((-1,-1)|(1,1))={u_1^2\over 2\theta+u_1+u_1^2}.$$
and for condition $\sigma_{\ell_0}=(-1,1)$:
\begin{equation}\label{ota}\mu^*((1,1)|(-1,1))={u_1\over \theta+(1+\theta)u_1+u_1^2}, \ \ 
\mu^*((-1,1)|(-1,1))={\theta\over \theta+(1+\theta)u_1+u_1^2},\end{equation}
$$\mu^*((1,-1)|(-1,1))={u_1^2\over \theta+(1+\theta)u_1+u_1^2}, \ \
\mu^*((-1,-1)|(-1,1))={\theta u_1\over \theta+(1+\theta)u_1+u_1^2}.$$

\begin{figure}[h]
	\begin{center}
		\includegraphics[width=7.3cm]{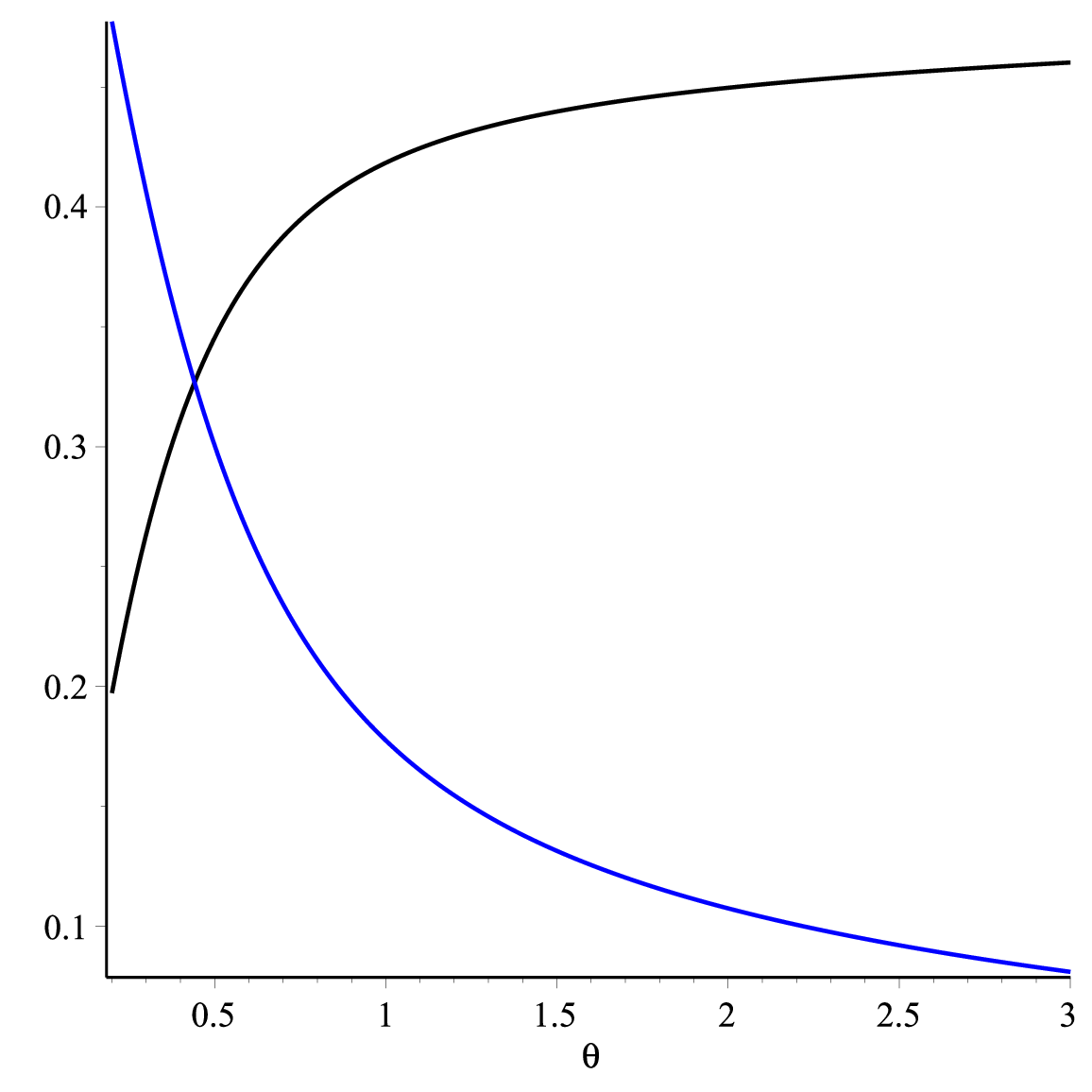}
	\end{center}
	\caption{The graph of  $\mu^*((1,1)|(1,1))$ (black) and $\mu^*((1,1)|(-1,1))$ (blue) at $a=0.3$ as functions of parameter $\theta$.}\label{muo}
\end{figure}
\begin{rk} 
From formulas (\ref{ona})-(\ref{ota}) (see also Figure \ref{muo}), it is evident that the system described by the Hamiltonian (\ref{Ha}) on a one-dimensional lattice, with the condition (\ref{lam}), has a unique equilibrium state corresponding to $u_1$. In this case, the conditional probability of predicting a hidden configuration (at the endpoints of each edge of the 1D tree) depends on the conditioned observed configuration. For instance, the hidden configuration with the highest $\mu^*$-probability coincides with the (conditioned) observed configuration.
\end{rk}	
{\bf Case	$k=2$, $a=b$, $c=1$}: In this case we consider $a=2$, $\theta=0.1$ and choose solution given in the second row of the Table: $u=x^2\approx 0.04$, $v=y^2\approx 1.6$, $w=z^2\approx 0.432$. Denote by $\mu_3$ the corresponding Gibbs measure, then for these values from (\ref{ms}) we get
$$\mu_3((1,1)|(1,1))\approx 0.347, \ \ \mu_3((-1,1)|(1,1))\approx 0.324, $$ $$ \mu_3((1,-1)|(1,1))\approx 0.324, \ \ \mu_3((-1,-1)|(1,1))\approx 0.005.$$
For condition $\sigma_{\ell_0}=(-1,1)$ we have  
$$\mu_3((1,1)|(-1,1))\approx 0.003, \ \ \mu_3((-1,1)|(-1,1))\approx 0.86, $$ $$ \mu_3((1,-1)|(-1,1))\approx 0.127, \ \ \mu_3((-1,-1)|(-1,1))\approx 0.01.$$ 

	Thus  with respect to measure $\mu_3$ the conditional probability of predicting a hidden configuration  depends on the conditioned observed configuration. For instance, the hidden configuration with the highest $\mu_3$-probability coincides with the (conditioned) observed configuration.

\subsection{ Applications to Machine Learning}

	Let us now discuss some relations of our results in machine learning:
	
	\begin{itemize} 
		
		\item In tasks where observed data (e.g., pixel intensities in images, word sequences) have intrinsic correlations, the term $ \sigma(x)\sigma(y) $ in the Hamiltonian (\ref{Ha}) allows the model to capture dependencies in the observations (as explained in the previous subsection) while inferring the hidden structure $ s $. For instance, in image denoising, the observed pixels $ \sigma $ are noisy, and the hidden spins $ s $ represent the clean image labels. The model learns to recover the true signal from noisy observations by leveraging the correlation between the noisy pixels and the hidden clean labels (e.g., \cite{Bu}).
		
		\item In weakly supervised learning, the Hamiltonian's mismatch penalty $(s(x)s(y) - \sigma(x)\sigma(y))$ enforces consistency between local predictions and global correlations in the observed data $\sigma$. This penalty term ensures that the model respects both the local relationships (captured by $s$) and the global structure in the data. It plays a critical role in training models where supervision is limited (e.g., \cite{Zh}).
		
		\item In graphical model learning, the couplings $J$ and emission parameters (embedded in $ p(\sigma|s) $) via contrastive divergence are essential for controlling the dynamics of the system related to the model. The penalty $(s(x)s(y) - \sigma(x)\sigma(y))$ depends on the sign of $J$. Specifically, for $ J > 0 $, the model penalizes configurations where $ s(x)s(y) < \sigma(x)\sigma(y) $, while for $ J < 0 $, the opposite holds, meaning the model favors configurations where $ s(x)s(y) > \sigma(x)\sigma(y) $. The tree structure allows for exact calculations of critical parameters (e.g., \cite{Wa}).
		
		\item In anomaly detection, outliers are detected by identifying configurations where $ s(x)s(y) $ deviates significantly from $ \sigma(x)\sigma(y) $. Such mismatches signal discrepancies between expected and observed correlations, which often correspond to unusual or anomalous data points (e.g., \cite{Ch}).
		
		\item Belief Propagation: On tree-like structures, belief propagation (or the sum-product algorithm) is exact. This means that the messages passed between neighboring vertices $ x $ and $ y $ in the tree structure can be used to compute marginals of the spin configuration at each vertex. Belief propagation updates the beliefs about the state of each spin based on the observed data and the messages received from neighboring vertices. The algorithm for the Ising model on a Cayley tree iterates over the tree, updating the probability of the spin at each vertex based on its neighbors:
		$$ \mu_{x \to y}(s(x)) = \sum_{s(y)} p(s(x) | s(y)) p(\sigma(x) | s(x)) \prod_{z\in \mathcal{N}(x) \setminus y} \mu_{z\to x}(s(z)), $$
		where $\mathcal{N}(x)$ is the set of all neighbors of $x$.
		
		This iterative process continues until convergence, and the marginal distribution at each vertex provides an estimate for the hidden spin configuration (e.g., \cite{Pe}, \cite{Ye}).
		
	\end{itemize}

	\section*{Data availability statements}
	The datasets generated during and/or analysed during the 
	current study are available from the corresponding author (U.A.Rozikov) on reasonable request.
	
	\section*{Conflicts of interest} The authors declare no conflicts of interest.

	\section*{ Acknowledgements}
	
	The authors gratefully acknowledge the University of Granada for	awarding the Visiting Scholar Grant (PPVS2024.04) to U.A. Rozikov. He also expresses his sincere gratitude to the DaSCI Institute and IMAG for their kind invitation and support during his academic visit.  \\
	We thank both reviewers for the critical and detailed evaluation. Especially in relation to the notion of a ``hidden Markov model.''


\begin{thebibliography}{99}

	
	\bibitem{Ar} S. Ariosto, \textit{Statistical Physics of Deep Neural Networks: Generalization Capability, Beyond the Infinite Width, and Feature Learning}, 2025, 10.48550/arXiv.2501.19281.
				
	\bibitem{Bah}	Y. Bahri, J. Kadmon,	J. Pennington,  S. Schoenholz, 	J. Sohl-Dickstein,  S. Ganguli, \textit{Statistical Mechanics of Deep Learning}, Annu. Rev. Condens. Matter Phys.  \textbf{11},  2020, 501-28.
		
\bibitem{Beh} F. Behrens, N. Mainali,  C. Marullo, S. Lee,  B. Sorscher, H. Sompolinsky, 
\textit{Statistical mechanics of deep learning}, J. Stat. Mech. (2024) 104007.
		
		\bibitem{Bis} C.M. Bishop, \textit{Pattern Recognition and Machine Learning}. Springer. 2006.
		
		\bibitem{BR} L.V. Bogachev, U.A. Rozikov, \textit{On the uniqueness of Gibbs measure in
			the Potts model on a Cayley tree with external field}. J. Stat. Mech. Theory Exp. \textbf{7} (2019), 073205, 76 pp.
		
	\bibitem{Bu}	A. Buades, B. Coll, and J-M Morel. \textit{A non-local algorithm for image
		denoising}. In Computer Vision and Pattern Recognition, 2005. CVPR 2005. IEEE	Computer Society Conference on, volume 2, pages 60-65. IEEE, 2005.
	
\bibitem{Ch} V. Chandola, A. Banerjee, V.  Kumar, \textit{Anomaly detection: A survey}. ACM Computing Surveys (CSUR), \textbf{41}(3), 2009, 1-58.

\bibitem{En} A.C.D. van Enter, V.N. Ermolaev, G. Iacobelli, C. K\"ulske, \textit{ Gibbs-non-Gibbs properties for evolving Ising models on trees}. Annales de l'I.H.P. Probabilit\'es et statistiques, \textbf{48}(3), 2012,  774-791.

\bibitem{FV} S. Friedli, Y. Velenik, \textit{Statistical mechanics of lattice systems. A concrete mathematical introduction}. Cambridge University Press, Cambridge, 2018.
		
		\bibitem{Gem} S.Geman, D. Geman,  \textit{Stochastic relaxation, Gibbs distributions, and the Bayesian restoration of images}. IEEE Transactions on Pattern Analysis and Machine Intelligence, \textbf{6}(6), (1984), 721-741.
		
		\bibitem{Ge} H.O. Georgii, \textit{Gibbs Measures and Phase Transitions},  Second edition. de Gruyter Studies in Mathematics, 9. Walter de Gruyter, Berlin, 2011.
		
	\bibitem{Gh}	Z.Ghahramani, M.I. Jordan, \textit{Factorial Hidden Markov Models}. Machine Learning. \textbf{29}, 1997, 245-273.
		
			\bibitem{HOR} F.H. Haydarov, B.A. Omirov, U.A. Rozikov, \textit{Coupled Ising-Potts Model: Rich Sets of Critical Temperatures and Translation-Invariant Gibbs Measures}, 	arXiv:2502.12014 [math.FA]
			
		

		\bibitem{Krz} F. Krzakala, M. M\'ezard, F. Sausset, Y. Sun, L. Zdeborov\'a, \textit{Statistical-physics-based reconstruction in compressed sensing}. Phys. Rev. X \textbf{2}(2), 2012, 021005.
		
		\bibitem{Krza} F. Krzakala, L. Zdeborov\'a, \textit{Statistical physics methods in optimization and machine learning}, https://sphinxteam.github.io/EPFLDoctoralLecture2021/Notes.pdf
		
				\bibitem{Ni} H. Nishimori, \textit{Statistical physics of spin glasses and information processing: an introduction}, Oxford University Press, 2001.
		
	\bibitem{Le} Y.	LeCun, S. Chopra, R. Hadsell, M. Ranzato,  F.  Huang,  \textit{A tutorial on energy-based learning}.  Predicting structured data MIT Press. 2006.
		

	\bibitem{Pe} J.	Pearl, \textit{Probabilistic reasoning in intelligent systems: Networks of plausible inference}. Morgan Kaufmann. 1988.
		
		\bibitem{Pra} V.V. Prasolov, {\it Polynomials} (Springer-Verlag Berlin
		Heidelberg, 2004)
		
		\bibitem{Qa} A. Qawasmeh, F. Mukhamedov, H.  Akın, \textit{Analysis of Gibbs Measures and Stability of
			 Dynamical System Linked to $(1,1/2)$-Mixed Ising Model on 
		$(m,k)$-Ary Trees}. Math. Phys. Anal. Geom. \textbf{28}, 10, 2025.
	
	\bibitem{QaM}	A. Qawasmeh, F. Mukhamedov,  O. Khakimov, \textit{On Gibbs measures of the Ising model on 
	$(k,m)$-ary trees}, Reviews in Mathematical Physics, 
\textbf{37}(05), 2450042, 2025.
		
		\bibitem{Ro} U.A.  Rozikov, \textit{Gibbs measures on Cayley trees}.  World Sci. Publ. Singapore. 2013.
		
		\bibitem{Rbp} U.A.	Rozikov, \textit{Gibbs measures in biology and physics: The Potts model}. World Sci. Publ. Singapore. 2023.
		
		\bibitem{Wa} M.J. Wainwright, M. I. Jordan,  \textit{ Graphical models, exponential families, and variational inference.} Foundations and Trends in Machine Learning, \textbf{1}(1-2), 2008, 1-305.
		
	\bibitem{We} 	Y. Weiss,  W T. Freeman, \textit{Correctness of Belief Propagation in Gaussian Graphical Models of Arbitrary Topology}, Advances in Neural Information Processing Systems 12 (NIPS 1999), 673-679.
		
	\bibitem{Ye} J.S. Yedidia, W. T. Freeman, Y.  Weiss,  \textit{Understanding belief propagation and its generalizations}. Exploring Artificial Intelligence in the New Millennium,  2005, 239-236.
		
	\bibitem{Zh} Z.-H. Zhou, \textit{A brief introduction to weakly supervised learning}, National Science Review, Volume 5, Issue 1, January 2018, Pages 44-53.
		
		
	\end{thebibliography}
	\end{document}